%% file: iclr-main_page.tex
\documentclass{article} 
\usepackage{xxxx_conference,times}

\usepackage{hyperref}
\usepackage{url}

\usepackage[utf8]{inputenc} 
\usepackage[T1]{fontenc}    
\usepackage{amsthm} 

\newtheorem{theorem}{Theorem}

\usepackage{booktabs}       
\usepackage{amsfonts}       
\usepackage{nicefrac}       
\usepackage{microtype}      
\usepackage{xcolor}         

\usepackage{amsmath}

\usepackage{lipsum}
\usepackage{multirow}
\usepackage{stmaryrd}

\usepackage{multirow}
\usepackage{bbm}
\usepackage{adjustbox}
\usepackage{xspace}

\usepackage{graphicx}
\usepackage{caption}
\usepackage{subcaption}
\usepackage{color, colortbl}
\usepackage{tablefootnote}
\usepackage{enumitem}  
\usepackage{lipsum}
\usepackage{xspace}

\usepackage{makecell}
\usepackage{algorithmic}
\usepackage{algorithm}
\usepackage{footnote}
\makesavenoteenv{tabular}
\usepackage{wrapfig}   
\usepackage{utfsym}
\usepackage{tabularx} 
\usepackage[breakable]{tcolorbox} 
\usepackage{listings}             

\input{defs}

\setlist{topsep=1pt,itemsep=1pt,partopsep=1pt, parsep=1pt}

\title{From Reward Shaping to Q-Shaping: Achieving Unbiased Learning with LLM-Guided Knowledge}
\author{Xiefeng Wu \\
Wuhan University \\
\texttt{wuxiefeng@whu.edu.cn} \\
}

\iclrfinalcopy 
\begin{document}

\maketitle

\begin{abstract}
Q-shaping is an extension of Q-value initialization and serves as an alternative to reward shaping for incorporating domain knowledge to accelerate agent training, thereby improving sample efficiency by directly shaping Q-values. This approach is both general and robust across diverse tasks, allowing for immediate impact assessment while guaranteeing optimality. We evaluated Q-shaping across 20 different environments using a large language model (LLM) as the heuristic provider. The results demonstrate that Q-shaping significantly enhances sample efficiency, achieving a \textbf{16.87\%} improvement over the best baseline in each environment and a \textbf{253.80\%} improvement compared to LLM-based reward shaping methods. These findings establish Q-shaping as a superior and unbiased alternative to conventional reward shaping in reinforcement learning.
\end{abstract}

\section{Introduction} \label{sec:intro}

\input{sections/introduction}

\section{Related Work} \label{sec:related_work}

\input{sections/related_work}

\section{Notation}\label{sec:notation}

\input{sections/notation}

\section{Q-shaping Framework}\label{sec:q-shaping}
\input{sections/q-shaping}

\section{Experiment Settings}\label{sec:experiment}

\input{sections/experiments}

\section{Results and Analysis}\label{sec:results}

\input{sections/results_analysis}

\section{Conclusion}\label{sec:conclusion}
\input{sections/conclusion}

\newpage
\bibliography{xxxx_conference}
\bibliographystyle{xxxx_conference}

\appendix

\end{document}

%% file: defs.tex
\newcommand{\mdp}{\mathcal{M}}
\newcommand{\sspace}{\mathcal{S}}
\newcommand{\aspace}{\mathcal{A}}
\newcommand{\saspace}{\mathcal{Z}}
\newcommand{\pispace}{\Pi}
\newcommand{\rawrew}{\mathcal{R}}
\newcommand{\rawerew}{\mathbf{r}}
\newcommand{\rawtrns}{P}
\newcommand{\rew}{\rawrew}
\newcommand{\erew}{\rawerew}
\newcommand{\trns}{\rawtrns}

\newcommand{\erewm}{\rawerew_{\mdp}}
\newcommand{\trnsm}{\rawtrns_{\mdp}}
\newcommand{\val}{\mathbf{v}}
\newcommand{\qval}{\mathbf{q}}

\newcommand{\act}{A}
\newcommand{\actpi}{\act^\pi}

\newcommand{\valpi}{\val^{\pi}}

\newcommand{\bell}{\mathcal{B}}

\newcommand{\optp}{\pi^{*}}

\newcommand{\dist}{{Dist}}

\newcommand{\E}{\mathbb{E}}

%% file: sections/introduction.tex
Reinforcement learning (RL) can solve complex tasks but often faces sample inefficiency. For example, AlphaGo~\citep{silver2016mastering_alpha_go} required approximately 4 weeks of training on 50 GPUs, learning from 30 million expert Go game positions to reach a 57\% accuracy. Similarly, training a real bipedal soccer robot required \(9.0 \times 10^8\) environment steps, amounting to 68 hours of wall-clock time for the full 1v1 agent~\citep{haarnoja2024learning_google_soccer}. These cases demonstrate the significant computational demands of RL.

To improve efficiency, popular methods include (1) imitation learning, (2) residual reinforcement learning, (3) reward shaping, and (4) Q-value initialization. Yet, each has limitations: imitation learning requires expert data, residual RL needs a well-designed controller, and Q-value initialization demands precise estimates. Therefore, reward shaping is the most practical approach, as it avoids the need for expert trajectories or predefined controllers.
    \vspace{-5pt}
\begin{figure}[h]
    \centering
    \includegraphics[width=\linewidth]{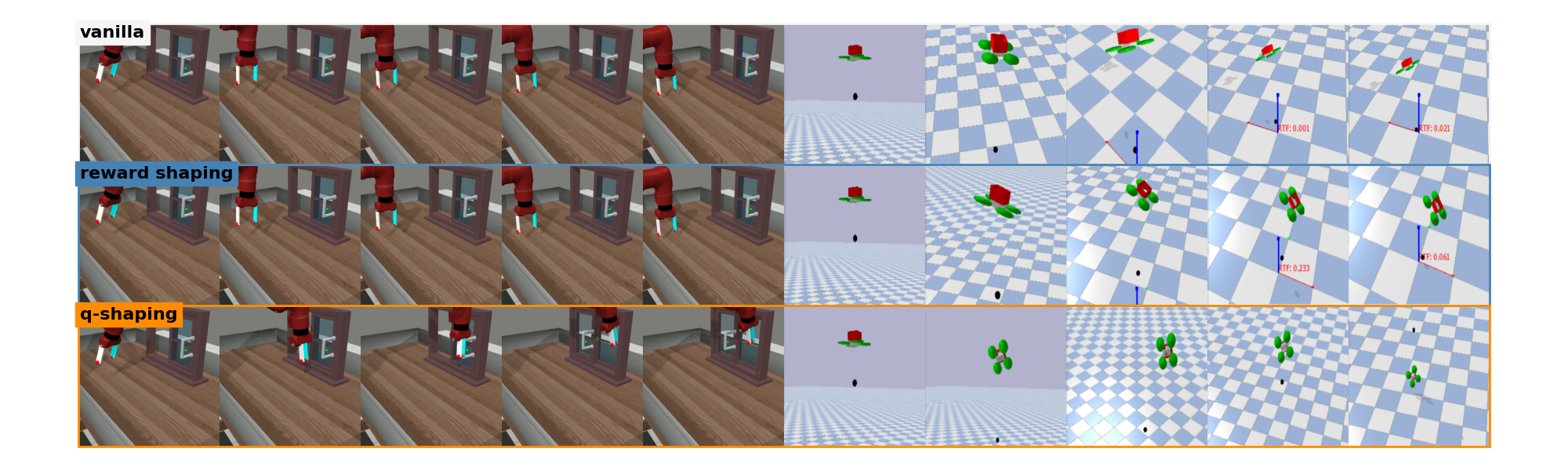}
    \caption{Agent behavior across different algorithms. Q-shaping impacts agent behavior quickly, enabling rapid evolution and improvement in the quality of heuristic functions. Vanilla refers to traditional RL algorithms, while reward shaping-enhanced RL algorithms cannot immediately impact agent behavior and have a slow verification period.}

    \vspace{-5pt}
    \label{fig:core_meaning_q_shaping}
\end{figure}

Reward shaping methods fall into two main categories: (1) potential-based reward shaping (PBRS)~\citep{rewardshaping_ang_wu} and (2) non-potential-based reward shaping (NPBRS). PBRS provides state-based heuristic rewards, while NPBRS extends to state-action pairs but lacks guarantees of optimality. Additionally, reward shaping methods often suffer from a slow verification process, requiring completion of training to assess the impact of the heuristic reward, which limits their development, as noted by Ma et al. (2023). Lastly, designing high-quality reward functions remains a challenging and often frustrating task for researchers, hindering the adoption of these methods \citep{ma2023eureka}.

With the growing popularity of large language models (LLMs), LLM-guided reinforcement learning (RL) has emerged as a promising field. This approach leverages the strong understanding capabilities of LLMs to guide RL agents in exploration or policy updates. Existing research has focused on two main areas: LLM-based policy generation and LLM-guided reward design. For example, \cite{chen2021decision, micheli2022transformers} utilize LLMs to enhance policy decisions, while \cite{kwon2023reward, carta2023grounding, ma2023eureka} employ LLMs to design reward structures. Although these works have improved task success rates, the challenges associated with reward shaping remain unresolved.

In this work, we introduce a novel framework, Q-shaping, which leverages domain knowledge from large language models (LLMs) to guide agent exploration. Unlike reward shaping, Q-shaping extends Q-value initialization by directly modifying Q-values at any training step without affecting the agent's optimality upon convergence. More importantly, Q-shaping enables rapid verification of heuristic guidance, allowing experimenters to refine the heuristic function efficiently. Additionally, Q-shaping is less dependent on the quality of the LLM, as the provided heuristic values do not alter the agent's optimality after convergence. Figure \ref{fig:core_meaning_q_shaping} illustrates the agent behavior across different algorithms.

In the "Q-shaping Framework" section, we provide a detailed analysis and supporting theorems demonstrating why Q-shaping preserves optimality and how imprecise Q-values can guide exploration to improve sample efficiency. In the experimental section, we employ GPT-4o as the heuristic provider and compare Q-shaping against popular baselines. The results indicate that Q-shaping achieved an average improvement of 16.87\% over the best baseline for each task across 20 different tasks. Additionally, we compare Q-shaping with LLM-guided reward shaping methods, such as T2R and Eureka, revealing that these methods experience a peak performance loss of 253.80\% in optimality compared to Q-shaping when aiming to improve task success rates.

%% file: sections/related_work.tex
\subsection{Heuristic Reinforcement Learning}
There are four common approaches to incorporating domain knowledge into reinforcement learning to enhance sample efficiency: (1) Imitation Learning, (2) Residual Policy, (3) Reward Shaping, and (4) Q-value Initialization.

Imitation Learning requires access to expert trajectories, as demonstrated by works such as GAIL~\citep{ho2016generative_gail}, where agents learn by mimicking expert behavior. However, the reliance on high-quality expert data limits its applicability in complex tasks. Residual Policy~\citep{johannink2019residual} methods involve designing a controller to guide agent actions, but this manual design process restricts their scalability and generality.

Q-value initialization, although promising, often requires precise Q-value estimates to derive an effective policy. For instance, Cal-QL~\citep{nakamoto2024cal_ql} employs calibrated Q-values to enhance agent exploration, but these calibrated values still rely on expert knowledge, making Q-value design more challenging than reward shaping. Consequently, few studies have pursued this direction due to the inherent difficulty in obtaining accurate Q-values compared to reward shaping.

Reward shaping directly modifies the reward function to influence agent behavior, improving training efficiency without requiring expert trajectories or manual controller design. This approach has been refined to address diverse learning scenarios, such as in Inverse Reinforcement Learning (IRL) \citep{ziebart2008maximum, wulfmeier2015maximum, finn2016guided} and Preference-based RL \citep{christiano2017deep, ibarz2018reward, lee2021pebble, park2022surf}. Additionally, various heuristic techniques have been introduced, including unsupervised auxiliary task rewards \citep{jaderberg2016reinforcement}, count-based reward heuristics \citep{bellemare2016unifying, ostrovski2017count}, and self-supervised prediction error heuristics \citep{pathak2017curiosity, stadie2015incentivizing, oudeyer2007intrinsic}.

However, reward shaping often suffers from inaccuracies in the heuristic functions and a slow verification process, which limits its effectiveness in certain applications.

\subsection{LLM\textbackslash VLM Agent}
LLMs/VLMs can achieve few-shot or even zero-shot learning in various contexts, as demonstrated by works such as Voyager \citep{wang2023voyager}, ReAct \citep{yao2022react}, and SwiftSage \citep{lin2024swiftsage}.In the field of robotics, VIMA \cite{jiang2022vima} employs multimodal learning to enhance agents' comprehension capabilities. Additionally, the use of LLMs for high-level control is becoming a trend in control tasks \cite{shi2024yell,liu2023interactive,ouyang2024long}.In web search, interactive agents \cite{gur2023real,shaw2024pixels,zhou2023webarena} can be constructed using LLMs/VLMs. Moreover, frameworks have been developed to reduce the impact of hallucinations, such as decision reconsideration \citep{yao2024tree, long2023large}, self-correction \citep{shinn2023reflexion, kim2024language}, and observation summarization \citep{sridhar2023hierarchical}.
\subsection{LLM-enhanced RL}
Relying on the understanding and generation capabilities of large models, LLM-enhanced RL has become a popular field\cite{du2023guiding,carta2023grounding}.
Researchers have investigated the diverse roles of large models within reinforcement learning (RL) architectures, including their application in reward design \cite{kwon2023reward,wu2024read,carta2023grounding,chu2023accelerating,yu2023language,ma2023eureka},
information processing \cite{paischer2022history,paischer2024semantic,radford2021learning},
as a policy generator,
and as a generator within large language models (LLMs)\cite{chen2021decision,micheli2022transformers,robine2023transformer,chen2022transdreamer}.
While LLM-assisted reward design has improved task success rates~\citep{ma2023eureka,xietext2reward}, it often introduces bias into the original Markov Decision Process (MDP) or fails to provide sufficient guidance for complex tasks. Additionally, the verification process is time-consuming, which slows down the pace of iterative improvements.

%% file: sections/notation.tex
\paragraph{Markov Decision Processes.}
We represent the environment as a Markov Decision Process (MDP) in the standard form: \(\mdp := \langle \sspace, \aspace, \rew, \trns, \gamma, \rho \rangle\). Here, \(\sspace\) and \(\aspace\) denote the discrete state and action spaces, respectively. We use \(\saspace := \sspace \times \aspace\) as shorthand for the joint state-action space. The reward function \(\rew \colon \saspace \to \dist([0,1])\) maps state-action pairs to distributions over the unit interval, while the transition function \(\trns \colon \saspace \to \dist(\sspace)\) maps state-action pairs to distributions over subsequent states. Lastly, \(\rho \in \dist(\sspace)\) represents the distribution over initial states. We denote \(\erewm\) and \(\trnsm\) as the true reward and transition functions of the environment.

For policy definition, the space of all possible policies is denoted as \(\pispace\). A policy \(\pi: \mathcal{S} \rightarrow \Delta(\mathcal{A})\) defines a conditional distribution over actions given states. A deterministic policy \(\mu: \mathcal{S} \rightarrow \mathcal{A}\) is a special case of \(\pi\), where one action is selected per state with a probability of 1. We define the value function as \(v \colon \Pi \to \mathcal{S} \to \mathbb{R}\) or \(q \colon \Pi \to \mathcal{S} \times \mathcal{A} \to \mathbb{R}\), both with bounded outputs. The terms \(\qval\) and \(\val\) represent discrete matrix representations, where \(\val(s)\) and \(\qval(s, a)\) specifically denote the outputs of an arbitrary value function for a given policy at a particular state or state-action pair.

An \textit{optimal policy} for an MDP \(\mdp\), denoted by \(\optp_\mdp\), is one that maximizes the expected return under the initial state distribution: \(\optp_\mdp := \arg\max_{\pi} \E_{\rho} [\valpi_\mdp]\). The state-wise expected returns of this optimal policy are represented by \(\val_\mdp^{\optp_\mdp}\). The Bellman consistency equation for the MDP \(\mdp\) at \(\mathbf{x}\) is given by \(\bell_\mdp(\mathbf{x}) := \erew + \gamma \trns \mathbf{x}\). Notably, \((\valpi_\mdp)^*\) is the unique vector that satisfies \((\valpi_\mdp)^* = \actpi \bell_\mdp((\valpi_\mdp)^*)\).We abbreviate \(\qval^*\) as \(\bigl(\qval_{\mdp}^{\pi_\mdp^*} \bigr)^*\) and \(\qval_\xi^*\) as \(\bigl(\qval_{\xi}^{\pi_\xi^*} \bigr)^*\) for some MDP \(\xi\).

\paragraph{Datasets} 

We define fundamental concepts essential for fixed-dataset policy optimization. Let \( D := \{\langle s, a, r, s' \rangle \}^d \) represent a dataset of \( d \) transitions. From this dataset, we can construct a local MDP \(\mathcal{D}\) and derive a local optimal Q-value function, denoted as \( q^*_D \).

Within the Q-shaping framework, let $\hat{\qval}$ denote the Q-function learned from TD estimation and Q-shaping. The LLM outputs are categorized into two types: \textit{goodQ}, which encourages exploration, and \textit{badQ}, which discourages it. Let \( G_{LLM} := \{(s, a, Q) \mid Q > 0 \}^d \) represent the dataset of \( d \) heuristic pairs focused on encouraging agent exploration. Similarly, \( B_{LLM} := \{(s, a, Q) \mid Q \le 0 \}^d \) denotes the dataset of \( d \) heuristic pairs aimed at preventing exploration. The complete collection of LLM outputs is given by \( D_{LLM} := \{G_{LLM}, B_{LLM}\} \).

\paragraph{Convergence} An agent is considered to have converged when it reaches 80\% of the peak performance. The peak performance is defined as the highest performance achieved by any of the baseline methods.

%% file: sections/q-shaping.tex
In the Q-learning framework, an experience buffer \(D\) is used to store transitions from the Markov Decision Process (MDP), supporting both online and offline training. The TD-update method utilizes this experience buffer to estimate the Q-values for (s, a) pairs. The policy is then derived from the trained Q-function, which maximizes $\qval(s,\cdot)$. Thus, accurate Q-value estimation is crucial, as it determines policy quality and guides exploration. To facilitate better exploration, Q-shaping leverages both the experience buffer and a heuristic function provided by a large language model to estimate the Q-function. The general form of Q-shaping is given by:

\[
\hat{\qval}^{k+1}(s, a) \leftarrow \hat{\qval}^{k}(s, a) + \alpha \hat{\qval}^k_{TD}(s, a) + h(s, a), \quad (s, a, h(s, a)) \in D^k_{LLM},
\]

where $\hat{\qval}^k_{TD}(s,a)$ represents the temporal-difference (TD) update estimation of $\qval(s,a)$ at step $k$, expressed as: $\hat{\qval}^k_{TD}(s,a) = r(s,a,s')+ \gamma \hat{\qval}^k(s,a)$.
Here, $D^k_{LLM}$ denotes the set of \((s, a, Q)\) pairs provided by the LLM at iteration \( k \).

In the early stages of training, the convergence of the Q-function does not yield optimal performance, as the agent has yet to gather high-quality trajectories. Previous works, such as MCTS~\citep{browne2012survey_mcts} and SAC~\citep{sac}, have employed action-bonus heuristics to bias Q-values, thereby facilitating better exploration. While these methods may compromise the accuracy of Q-value estimation, they significantly enhance the agent's trajectory exploration in the short term. Our approach aligns with these action-bonus methods but leverages the LLM's understanding and thinking abilities to provide heuristic bonuses, resulting in a more informed exploration strategy.

\subsection{Unbiased Optimality}
The Q-value represents a high-level abstraction of both the environment and the agent's policy. It encapsulates key elements such as rewards $r$, transition probabilities $P$, states $s$, actions $a$, and the policy $\pi$, thereby integrating the environmental dynamics and the policy under evaluation.
Changes in any of these components directly influence the Q values associated with different actions. Specifically, the term $\mathbf{h}$ can take various forms, such as the entropy term used in SAC or the UCT heuristic term employed in MCTS and is utilized to shape the Q-values at each step.
Compared to these algorithms, the LLM-guided Q-shaping method provides heuristic guidance only at specific steps, ensuring that the final optimality of the Q-function remains unaffected. The converged shaped Q-function is thus equivalent to the locally optimal Q-function $\hat{\qval}$:

\begin{theorem}[Contraction and Equivalence of $\hat{\qval}$]
\label{theorem:Q_contraction_Q_equal}
Let $\hat{\qval}$ be a contraction mapping defined in the metrics space $(\mathcal{X},\|\cdot\|_{\infty})$, i.e,
$$\|   \bell_{\mathcal{D}}(\hat{\qval}) -\bell_{\mathcal{D}}(\hat{\qval}')  \|_{\infty} \le \gamma \|\hat{\qval} - \hat{\qval}' \|_{\infty}  $$, where $\bell_{\mathcal{D}}$ is the Bellman operator for the sampled MDP \(\mathcal{D}\) and \(\gamma\) is the discount factor.

Since both $\hat{\qval}$ and $\qval$ are updated on the same MDP,
we have the following equation:
$$\hat{\qval}_\mathcal{D}^{*} = \qval^{*}_\mathcal{D}$$
\label{thm:qhat_is_contraction}
\end{theorem}

\begin{proof}
See Appendix.
\end{proof}

\subsection{Utilizing Imprecise Q value Estimation}
At the early training stage, the Q-values for different actions are nearly identical, leading the policy to execute actions randomly.To address this, we leverage the LLM's domain knowledge to provide positive Q-values for actions that contribute to task success and negative Q-values for actions that do not.
The imprecise Q-values provided by the LLM can be categorized into two types: overestimations and underestimations

\paragraph{Underestimation of Non-Optimal Actions}
An agent does not need to fully traverse the entire state-action space to identify the optimal trajectory that leads to task success. Therefore, imprecise Q-value estimation can be effectively utilized to guide the agent's exploration.

For instance, consider a scenario where the agent is required to control a robot arm to operate on a drawer located in front of it. In this case, actions such as moving the arm backward or upward are evidently unhelpful in finding the optimal trajectory. Assigning very low Q-values to these non-contributory actions discourages the agent from exploring them, thereby enhancing sample efficiency.

\begin{wrapfigure}{r}{0.58\linewidth}
\vspace{-10pt} 
\vspace{-10pt} 
\begin{minipage}{0.58\textwidth}
\begin{algorithm}[H]
\caption{Q-shaping}
\label{algo:q-shaping}
\begin{algorithmic}[1]
\small
\STATE \textbf{Require}: Good Q-set $G_{llm}$, Bad Q-set $B_{llm}$ provided by the LLM, RL solver $\mathcal{A}$
\STATE \textbf{Goal}: Compute the average performance over 10 runs
\STATE \textbf{Initialize}: Start 20 agents $\{\text{Agent}_1, \text{Agent}_2, \dots, \text{Agent}_{20}\}$
\STATE \textcolor{gray}{\texttt{\# for each agent, do:}}
\STATE agent.explore(steps = 5000)
\STATE \textcolor{gray}{\texttt{\# Apply Q-shaping and Policy-shaping}}
\STATE agent.q\_shaping($G_{llm}$, $B_{llm}$)
\STATE agent.policy\_shaping($G_{llm}$, $B_{llm}$)
\STATE \textcolor{gray}{\texttt{\# Further exploration}}
\STATE agent.explore(steps = 10000)
\STATE \textcolor{gray}{\texttt{\# Synchronize agents}}
\STATE agent.wait()
\STATE \textcolor{gray}{\texttt{\# Remove 10 lower-performing agents}}
\STATE agent.remove\_if\_latter()
\STATE \textcolor{gray}{\texttt{\# Continued exploration and training}}
\STATE agent.explore\_and\_train()
\STATE \textbf{Output}: Average performance over 10 runs
\end{algorithmic}
\end{algorithm}
\end{minipage}
\vspace{-10pt} 
\vspace{-10pt} 

\end{wrapfigure}

\paragraph{Overestimation of Near-Optimal Actions}
At the initial training phase (iteration step \( k = 0 \)), let action \( a \) be assumed to have the highest estimated Q-value for a given state \( s \), while \( a^* \) denotes the true optimal action. This assumption leads to the inequality \( \hat{\qval}(s, a^*) < \hat{\qval}(s, a) < \qval^*(s, a^*) \). Consequently, the agent is predisposed to explore actions around the suboptimal \( a \) in its search for states, given that \( \mu(s) = \max_a \hat{\qval}(s, \cdot) + \epsilon \), where \( \epsilon \sim \mathcal{N}(0, \delta^2) \) .

However, the number of steps required to discover the optimal action \(a^*\) is inherently constrained by the environment and the distance between \(a\) and \(a^*\). To expedite this exploration process, we introduce an action \(a_{LLM}\) suggested by the LLM, replacing \(a\) via Q-shaping guided by the loss function in Equation \ref{eq:q-shaping-loss} to enhance sample efficiency. Given the assumption \(|a_{LLM} - a^*| < |a - a^*| < \delta\), we can express \(\mu(s) = a_{LLM} + \epsilon\). Consequently, the agent has a higher chance of selecting $a^*$, significantly improving the likelihood of identifying the optimal trajectory.

In conclusion, by letting the LLM provide the goodQ set and badQ set, the agent is guided to prioritize exploring actions suggested by the LLM, thereby enhancing sample efficiency. Over time, as indicated by \cite{NIPS2010_091d584f,td3} and Theorem \ref{thm:qhat_is_contraction}, \(\hat{\qval}\) converges towards the locally optimal Q-function. We now present the theoretical upper bound on the sample complexity required for \(\hat{\qval}\) to converge to \(\qval^*_\mathcal{D}\) for any given MDP \(\mathcal{D}\):

\begin{theorem}[Convergence Sample Complexity]
\label{thm:sample complexity}
The sample complexity $n$ required for $\hat{\qval}$ to converge to the local optimal fixed-point $\qval^*_D$ with probability $1-\delta$ is:
$$n > \mathcal{O}\left( \frac{|S|^2}{2\epsilon^2}\ln{\frac{2|S \times A|}{\delta}} \right)$$

\end{theorem}

\begin{proof}
See proof at appendix.
\end{proof}

Theorem \ref{thm:sample complexity} establishes an upper bound on the sample complexity, indicating that the imprecise Q-values provided by the LLM will be corrected within a finite number of steps. Therefore, any heuristic values can be introduced during the early training iterations, and the Q-shaping framework will adapt to inaccurate Q-values over time.

\subsection{Algorithm Implementation}
For the implementation of Q-shaping, we employ TD3~\citep{td3} as the RL solver (backbone) and GPT-4o as the heuristic provider, introducing three additional training phases: (1) Q-Network Shaping (2) Policy-Network Shaping, and (3) High-performance agent selection. Pseudo-code \ref{algo:q-shaping} outlines the detailed steps of the Q-shaping framework.

\paragraph{Q-Network Shaping} In the Q-shaping framework, the LLM is tasked with providing a set of \((s, a, Q)\) pairs to guide exploration. This approach is particularly crucial during the early training stage when it is challenging for the agent to independently discover expert trajectories. Traditional RL solvers often require a substantial number of steps to identify the correct path to success, leading to sample inefficiency. The goal of the Q-shaping framework is to leverage the provided \((s, a, Q)\) pairs to accelerate exploration and help the agent quickly identify the optimal path.

\begin{figure}[h]
    \centering
    \includegraphics[width=\linewidth]{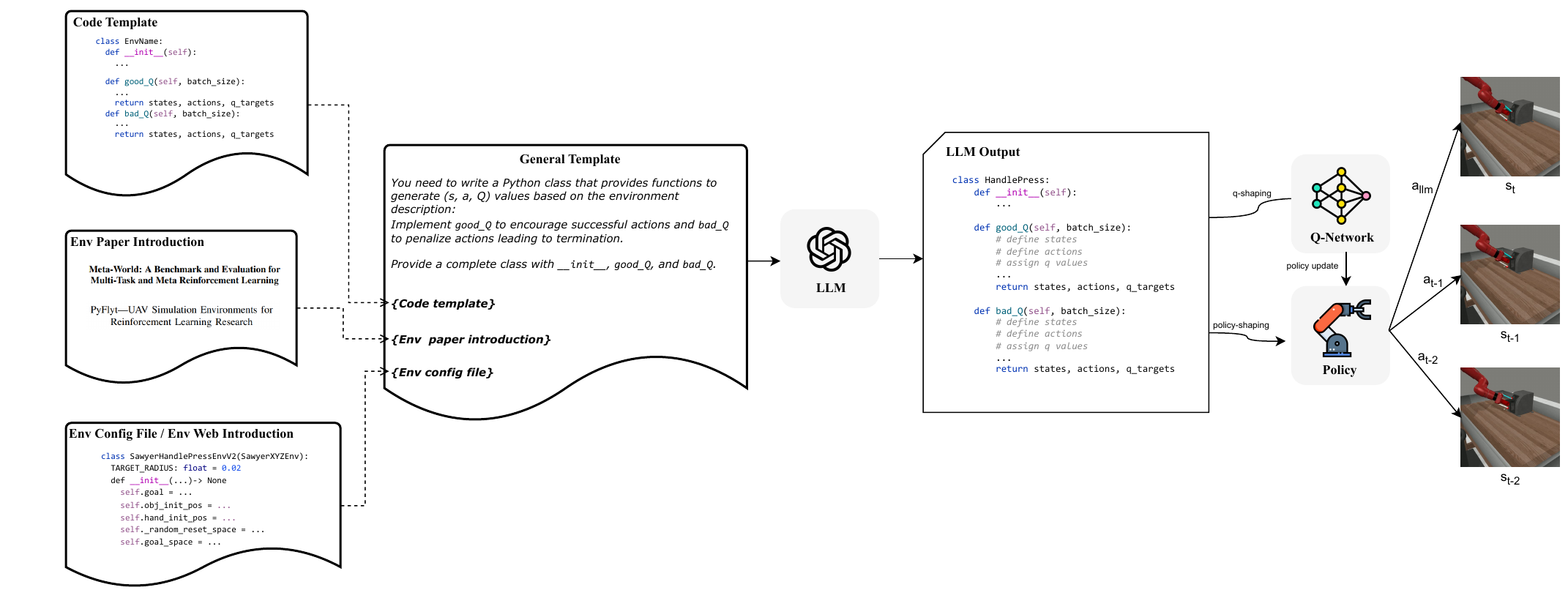}
    \caption{Q-shaping prompt. There is a general code template that specifies the required structure for the generated code. In addition to the template, three key pieces of information are necessary to generate an effective heuristic function: the code template, an introduction to the environment provided in the paper, and the environment configuration file.}
    \vspace{-5pt}
    \label{fig:q_shaping_prompt}
\end{figure}

To obtain \( D_{LLM} \), we construct a general code template as the prompt as illustrated in Figure \ref{fig:q_shaping_prompt}, supplemented by task-specific environment configuration files and a detailed definition of the observation and action spaces within the simulator. Subsequently, we apply the loss function \( L_{q-shaping} \) to update the Q-function:

\begin{equation}
    \label{eq:q-shaping-loss}
     L_{q-shaping}(\theta) = E_{(s_i,a_i,Q_i) \sim D_g}{(Q_i - \hat{\qval}_{\theta}(s_i,a_i))^2}
\end{equation}

\paragraph{Policy-Network Shaping}

In most reinforcement learning (RL) algorithms, the policy is derived from the Q-function, where the policy is optimized to execute actions that maximize the Q-value given a state. The policy update is expressed as:$\mu(s) = \arg\max_a \hat{\qval}(s, \cdot)$

While introducing a learning rate and target policy can help stabilize the training process and prevent fluctuations in the policy network, this approach often slows down the convergence speed. To accelerate this adaptation, we introduce a "Policy-Network Shaping" phase designed to allow the policy to quickly align with the good actions and avoid the bad actions provided by the LLM.

The shaping loss function is defined as:
\begin{equation}
    L_{policy-shaping} = \lambda_1 \mathbb{E}_{(s, a) \sim G_{LLM}} \left[ \|\mu(s) - a \|^2 \right] - \lambda_2 \mathbb{E}_{(s, a) \sim B_{LLM}} \left[ \|\mu(s) - a \|^2 \right]
\end{equation}
where \((s, a) \sim G_{LLM}\) and \((s, a) \sim B_{LLM}\) represent state-action pairs sampled from the LLM-provided \textit{goodQ} and \textit{badQ} sets, respectively, and \(\lambda_1\) and \(\lambda_2\) are hyperparameters controlling the influence of the LLM-guided shaping.

With this "Policy-Network Shaping" phase, researchers can quickly observe the impact of heuristic values, facilitating the rapid evolution of heuristic quality, ultimately leading to a more efficient exploration process and faster convergence to optimal behavior.

\paragraph{High-Performance Agent Selection}

With Q-network shaping and policy-network shaping, the Q-shaping framework enables a more rapid verification of the quality of provided heuristic values compared to traditional reward shaping. This allows the experimenter to selectively retain high-performing agents for further training while discarding those that underperform. As outlined in Algorithm \ref{algo:q-shaping}, following the shaping of the policy and Q-values, each agent is allowed 10,000 steps to explore. After this exploration phase, weaker agents are removed, and only the top-performing agent continues with the training process.

%% file: sections/experiments.tex
We investigate the following \textbf{hypotheses} through a series of four experiments:
\begin{enumerate}
    \item Can Q-shaping enhance sample efficiency in reinforcement learning?
    \item Can Q-shaping adapt to incorrect or hallucinated heuristics while maintaining optimality?
    \item Does Q-shaping outperform LLM-based reward shaping methods?
    \item Can LLM design heuristic functions that provide s,a,Q altogether?
\end{enumerate}

\begin{figure}[h]
    \centering
    \includegraphics[width=\linewidth]{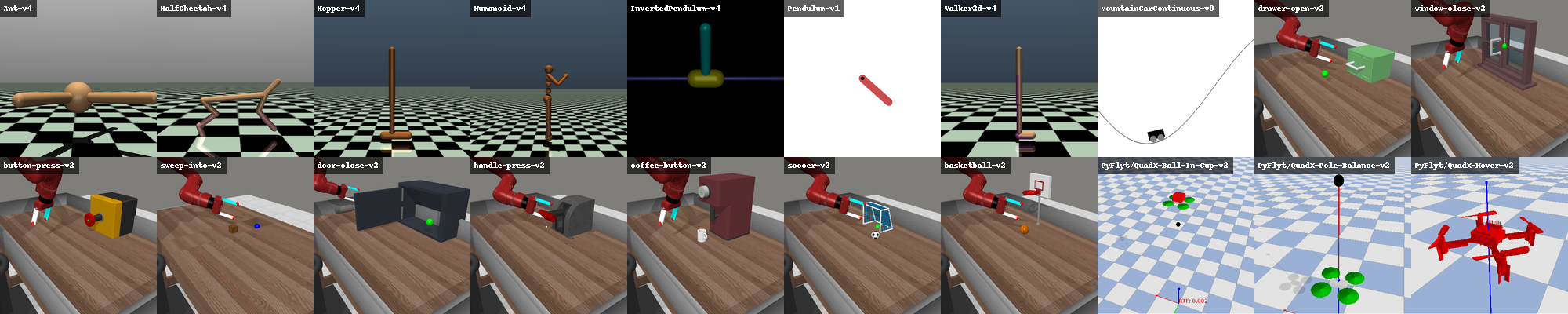}
    \caption{Evaluation Environments}
    \vspace{-5pt}
    \label{fig:environments}
\end{figure}

To validate these hypotheses, we conducted three primary experiments and one ablation study. GPT-4o served as the heuristic provider, while TD3 was employed as the reinforcement learning (RL) backbone, forming \textbf{LLM-TD3}. As illustrated in Figure \ref{fig:environments}, Q-shaping and various baseline methods were evaluated across 20 distinct tasks involving drones, robotic arms, and other robotic control challenges. Below, we describe the specific experiments and their objectives:

\begin{enumerate}
    \item \textbf{Sample Efficiency Experiment:} We compare Q-shaping with four baseline methods to assess its impact on sample efficiency.
    \item \textbf{Comparison with LLM-based Reward Shaping:} Q-shaping, which integrates domain knowledge to assist in agent training, is compared with Text2Reward and Eureka to evaluate its performance relative to existing LLM-based reward shaping approaches.
    \item \textbf{LLM Quality Evaluation:} Although Q-shaping guarantees optimality, its reliance on LLM-provided heuristics may influence performance. This experiment evaluates the quality of different LLM outputs.
    \item \textbf{Ablation Study on Q-shaping phases:} Q-shaping introduces three key training phases. This experiment isolates and examines the contribution of each phase to overall performance.
\end{enumerate}

\begin{figure}[h]
    \centering
    \includegraphics[width=\linewidth]{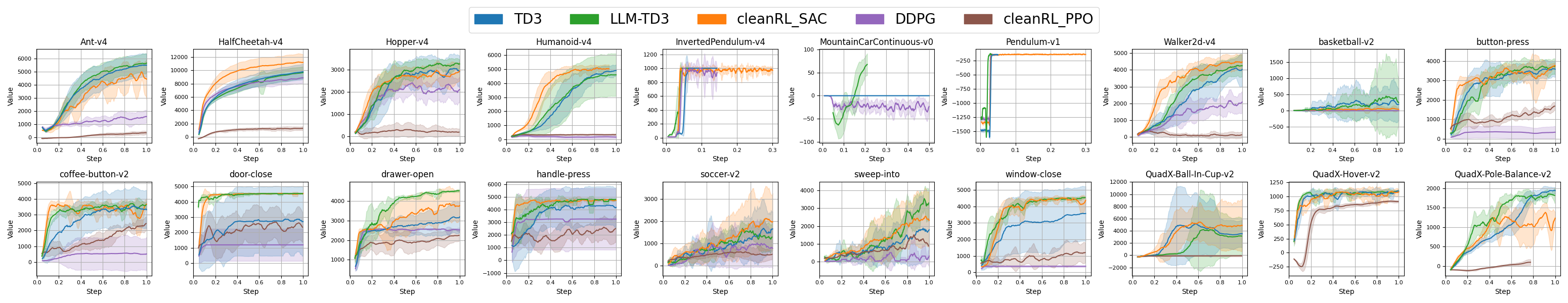}
    \caption{Learning curve comparison of each algorithm across 20 tasks.}

    \vspace{-5pt}
    \label{fig:sample_efficiency_learning_curve}
    \vspace{-5pt}
\end{figure}

\paragraph{Environments} We evaluate Q-shaping across 20 distinct environments, including 8 from Gymnasium Classic Control and MuJoCo~\citep{todorov2012mujoco}, 9 from MetaWorld~\citep{yu2020meta_world}, and 3 from PyFlyt~\citep{tai2023pyflyt}. The environments span a range of robot types, from simple pendulum systems to humanoid control. Notably, the robot arm and drone environments used are less commonly studied, making it unlikely that the LLM was pretrained on these specific environments.

\paragraph{Baselines} For the sample efficiency experiments, we compared Q-shaping against several baseline algorithms, including CleanRL-PPO, CleanRL-SAC~\citep{huang2022cleanrl}, DDPG~\citep{ddpg}, and TD3~\citep{td3}. When evaluating Q-shaping against other reward shaping methods, we selected Text2Reward and Eureka as baselines. In the LLM-type ablation study, we assessed the performance of different LLMs: O1-Preview, GPT-4o-Mini, Gemini-1.5-Flash~\citep{team2023gemini}, DeepSeek-V2~\citep{deepseekai2024deepseekv2strongeconomicalefficient}, and Yi-Large~\citep{ai2024yi_large}.

For the reward shaping method comparison, we implemented Eureka and Text2Reward~\citep{xietext2reward}. Specifically, for the MetaWorld tasks using Eureka, we set $K=16$ and limited the evolution round to 1 due to the long verification cycle of Eureka.

\paragraph{Metrics} To evaluate sample efficiency, we measure the number of steps required to reach 80\% of peak performance, where peak performance is defined as the highest performance achieved by any baseline agent. For clarity in visualization, improvements exceeding 150\% are truncated to 150\%. Each algorithm is tested 10 times, and the average evaluation performance is reported. Evaluations are conducted at intervals of 5,000 steps. During each evaluation, the agent will be tested over 10 episodes, and the average episodic return will be plotted to form the learning curve.

%% file: sections/results_analysis.tex
\vspace{-5pt}
\begin{figure}[h]
\vspace{-10pt}
    \centering
    \includegraphics[width=\linewidth]{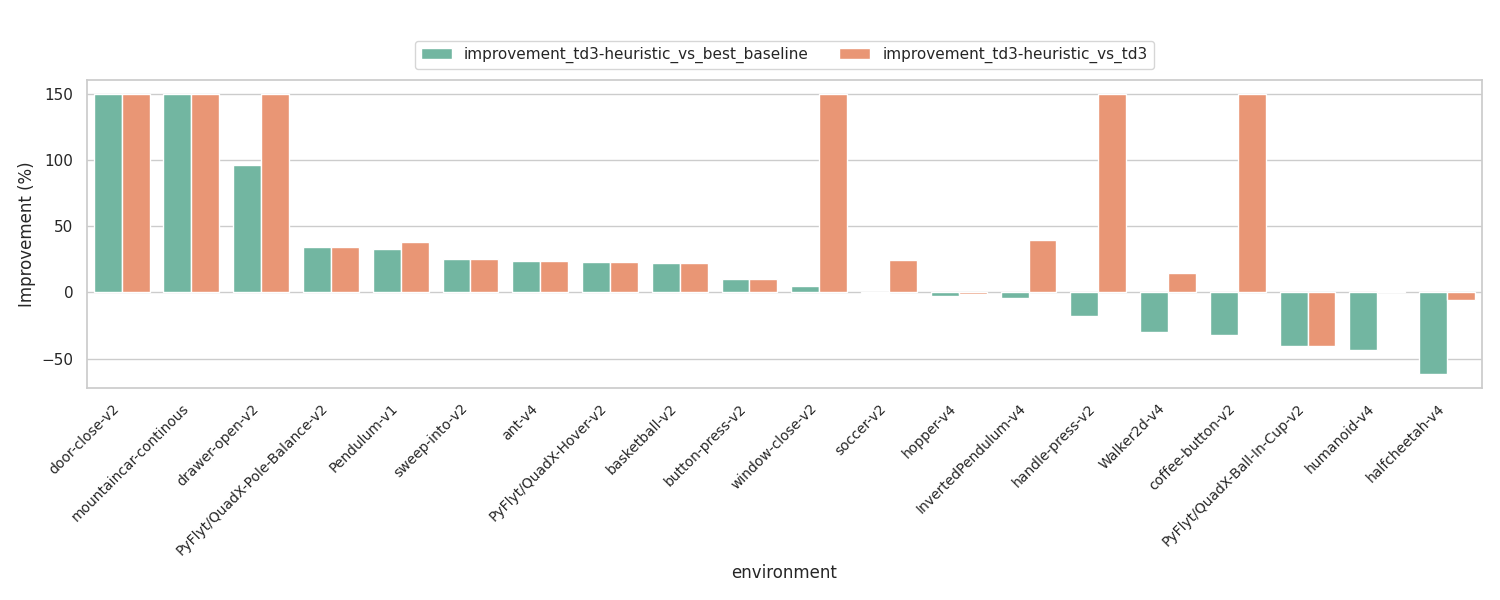}
    \caption{Q-shaping improvement over the best baseline in each environment and its improvement over TD3.}
    \vspace{-10pt}
    \label{fig:sample_efficiency_bar}
\end{figure}
\vspace{-5pt}

\paragraph{Q-Shaping Outperforms Best Baseline by an Average of 16.87\% Across 20 Tasks}
As shown in Figure \ref{fig:sample_efficiency_bar} and Figure \ref{fig:sample_efficiency_learning_curve}, Q-shaping demonstrated a notable improvement over both the best baseline and TD3 across 20 tasks. On average, Q-shaping improves performance by 16.87\% compared to the best baseline and by 55.39\% compared to TD3, highlighting its effectiveness in enhancing sample efficiency and task performance. This supports \textbf{H1}.

\vspace{-5pt}
\begin{figure}[h]
    \centering
    \includegraphics[width=\linewidth]{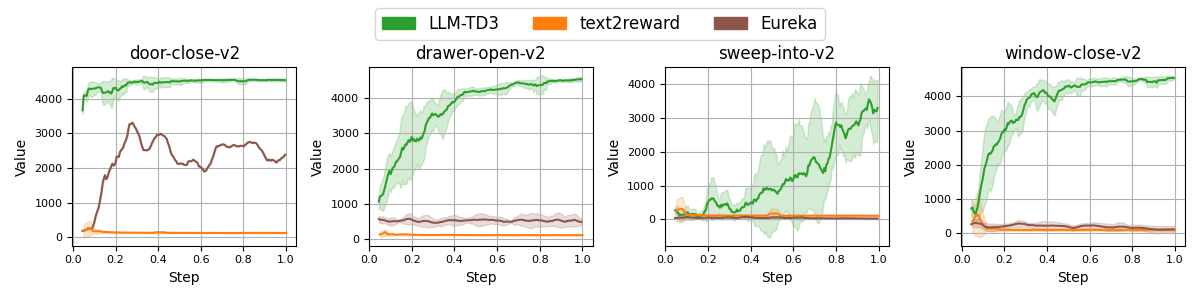}
    \caption{Learning curve comparison between Q-shaping and LLM-based reward shaping methods.}

    \vspace{-10pt}
    \label{fig:rewardshpaing}
\end{figure}
\vspace{-5pt}

\paragraph{Q-Shaping Outperforms LLM-Based Reward Shaping Methods by 253.80\%}
We evaluated Q-shaping and baseline methods on four Meta-World environments: \textit{door-close}, \textit{drawer-open}, \textit{window-close}, and \textit{sweep-into}. Using peak performance as the basis for comparison, Q-shaping achieved substantial improvements over both the Eureka and T2R baselines according to Figure \ref{fig:rewardshpaing}.

Compared to the best baseline, LLM-TD3 improved by 38.68\% in the \textit{door-close} task, 406.04\% in \textit{drawer-open}, 389.77\% in \textit{window-close}, and 180.70\% in \textit{sweep-into}, resulting in an average peak performance improvement of 253.80\%.

Though LLM-based reward shaping methods can improve task success rates~\citep{ma2023eureka,xietext2reward}, they often sacrifice optimality by modifying the original MDP. In contrast, Q-shaping achieves superior performance, retaining both success and optimality, with a 253.80\% improvement over the best LLM-based reward shaping methods.This supports \textbf{H2} and \textbf{H3}.

\paragraph{Most LLMs Can Provide Correct Heuristic Functions}

\begin{wraptable}{r}{0.6\textwidth} 
\vspace{-10pt} 
\captionsetup{justification=raggedright} 
\centering 
\caption{Evaluation of LLM Quality in Outputting Heuristic Values}
\label{llm-quality-evaluation}
\resizebox{\linewidth}{!}{ 
\begin{tabular}{@{}lccccc@{}}
\toprule
\textbf{Metric} & \textbf{o1-Preview} & \textbf{GPT-4o} & \textbf{Gemini} & \textbf{DeepSeek-V2.5} & \textbf{yi-large} \\ 
\midrule
Template Adherence (\%) & 100.0 & 100.0 & 40.0 & 100.0& 100.0 \\
Correct Q-values (\%)   & 100.0 & 100.0 & 60.0 & 100.0& 100.0 \\
Correct State-Action Dim (\%) & 100.0 & 100.0 & 80.0 & 100.0& 100.0 \\
Code Completeness (\%)  & 100.0 & 100.0 & 20.0 & 100.0& 100.0 \\
Bug-Free (\%)           & 100.0 & 100.0 & 20.0 & 100.0& 100.0 \\
Average (\%)            & 100.0 & 100.0 & 44.0 & 100.0& 100.0 \\
\bottomrule
\end{tabular}
\vspace{-10pt} 
\vspace{-10pt} 
\vspace{-10pt} 
\vspace{-10pt} 
}
\end{wraptable}

We evaluated the quality of LLM-generated heuristic functions from five perspectives: (1) adherence to the required code template, (2) correctness of the assigned Q-values, (3) accuracy of the state-action dimension, (4) completeness of the generated code, and (5) presence of bugs in the generated code. Each LLM was prompted 10 times with the same request, and we quantified their performance using a correctness rate across these metrics.

The results, as shown in Table \ref{llm-quality-evaluation}, indicate that most LLMs, including o1-Preview, GPT-4o, DeepSeek-V2.5, and yi-large, provided correct heuristic functions with a 100\% success rate across all evaluation metrics. However, Gemini exhibited poorer performance, achieving only 44\% on average. This supports \textbf{H4}.

\paragraph{Ablation Study on Additional Training Phases}
We conducted an ablation study to evaluate the impact of three key training phases: (1) Q-Network Shaping, (2) Policy-Network shaping, and (3) agent selection, across four Meta-World environments: \textit{door-close}, \textit{drawer-open}, \textit{window-close}, and \textit{sweep-into}. The effectiveness of each phase was measured by convergence steps, with algorithms marked as "Failed" if they did not reach the convergence threshold within $10^6$ steps. The study aimed to assess how each phase contributes to improving sample efficiency.

As shown in Table \ref{ablation-study}, each training phase significantly enhances sample efficiency. Q-Network shaping and policy-network shaping together result in substantial performance gains for TD3. Additionally, the agent selection phase helps by eliminating agents that fail to explore effective trajectories in the early stages of training, providing a slight improvement in average sample efficiency.

\begin{table}[h]
\vspace{-2pt} 
\caption{Ablation Study on Additional Training Phases}
\vspace{-2pt} 
\label{ablation-study}
\centering
\begin{tabularx}{1\textwidth}{@{}ccXXXX@{}}
\toprule
\multicolumn{2}{c}{Phase}  & \multicolumn{4}{c}{Environment}  \\ 
\cmidrule(r){1-2} \cmidrule(lr){3-6}
{\small (Q,Policy)-shaping} & {\small Selection}  & {\scriptsize door-close-v2} & {\scriptsize drawer-open-v2} & {\scriptsize sweep-into-v2} & {\scriptsize window-close-v2} \\ \midrule

$\times$ & $\times$   & Failed & Failed & Failed & 759999  \\
\checkmark & $\times$  & 30000 & 275000 & 860000 & 365000  \\
\checkmark & \checkmark & 25000 & 265000 & 790000 & 215000  \\
\bottomrule
\end{tabularx}
\vspace{-10pt} 
\end{table}

%% file: sections/conclusion.tex
We propose Q-shaping, an alternative framework that integrates domain knowledge to enhance sample efficiency in reinforcement learning. In contrast to traditional reward shaping, Q-shaping offers two key advantages: (1) it preserves optimality, and (2) it allows for rapid verification of the agent's behavior. These features enable experimenters or LLMs to iteratively refine the quality of heuristic values without concern for the potential negative impact of poorly designed heuristics. Experimental results demonstrate that Q-shaping significantly improves sample efficiency and outperforms LLM-guided reward shaping methods across various tasks.

We hope this work encourages further research into advanced techniques that leverage LLM outputs to guide and enhance the search process in reinforcement learning.